\documentclass{article}

\usepackage[preprint]{icbinb2021}

\usepackage[utf8]{inputenc} % allow utf-8 input
\usepackage[T1]{fontenc}    % use 8-bit T1 fonts
\usepackage{hyperref}       % hyperlinks
\usepackage{url}            % simple URL typesetting
\usepackage{booktabs}       % professional-quality tables
\usepackage{amsfonts}       % blackboard math symbols
\usepackage{nicefrac}       % compact symbols for 1/2, etc.
\usepackage{microtype}      % microtypography
\usepackage{xcolor}         % colors
\usepackage{amsthm,amsmath}
\usepackage{cleveref}
\usepackage{float}
\usepackage{graphicx}

\newcommand{\R}{\mathbb{R}}
\newtheorem{proposition}{Proposition}
\newtheorem{lemma}{Lemma}

\title{Barely Biased Learning for Gaussian Process Regression}

\author{
  David R.~Burt\textsuperscript{\textasteriskcentered} \\
  Department of Engineering \\
  University of Cambridge\\
  Cambridge, UK \\
  \small\texttt{drb62@cam.ac.uk} \\
  \And
  Artem Artemev\textsuperscript{\textasteriskcentered} \\
  Department of Computing \\
  Imperial College London\\
  London, UK \\
  \small \texttt{a.artemev20@imperial.ac.uk}
\And
  Mark van der Wilk\\
  Department of Computing \\
  Imperial College London\\
  London, UK \\
  \small\texttt{m.vdwilk@imperial.ac.uk}
}

\begin{document}

\maketitle

\begin{abstract}
Recent work in scalable approximate Gaussian process regression has discussed a bias-variance-computation trade-off when estimating the log marginal likelihood. We suggest a method that adaptively selects the amount of computation to use when estimating the log marginal likelihood so that the bias of the objective function is guaranteed to be small. While simple in principle, our current implementation of the method is not competitive computationally with existing approximations.
\end{abstract}

\section{Introduction}

Scalable hyperparameter selection for Gaussian process regression (GPR) is a frequently-studied problem in machine learning. Recent work on approximate empirical Bayes in conjugate GPR models has explored a bias-variance-computation trade-off \citep{pmlr-v139-artemev21a, pmlr-v139-potapczynski21a, wenger2021reducing}. Generally the methods used are at least partially adaptive, in the sense that they use a stopping criterion, which is (often indirectly) related to the bias or variance of the method, to decide how many iterations of a Krylov subspace method to run, which controls the amount of computation used. Our main practical insight is that with minor modifications, we can more directly link the stopping criterion to the quality of our estimate of the log marginal likelihood. This allows us to ensure that in each iteration of hyperparameter learning, we obtain an estimate of the log marginal likelihood which has an expectation within $\epsilon$ of the log marginal likelihood, where $\epsilon>0$ is a user-specified parameter. This is achieved by an application of Gauss-Radau quadrature to obtain an estimate of the log determinant of the kernel matrix, the expectation of which is a lower bound on the log marginal likelihood.

\section{Related work}

\Citet{Gibbs97efficientimplementation} proposed using the method of conjugate gradients (CG) \citep{hestenes1952methods} for estimating the gradients of the log marginal likelihood of Gaussian process regression. They advocated the use of upper and lower bounds on the entries of the gradients as a method for deciding on the number of iterations of CG to run. \Citet{davies2005effective} carried out a detailed study and discussion of this approach. \Citet{ubaru2017fast} considered estimating the log marginal likelihood directly by employing Lanczos quadrature \citep{golub1969calculation} and a stochastic trace estimator \citep{hutchinson1989stochastic}. \Citet{gardner2018gpytorch} provided a GPU-compatible implementation of the estimators developed in \citet{Gibbs97efficientimplementation} and \citet{ubaru2017fast}, and demonstrated advantages of this approach on modern computer architectures.

\Citet{pmlr-v139-artemev21a} proposed using a low-rank approximation for the log determinant, as is done in variational GP regression \citep{pmlr-v5-titsias09a}, in conjunction with the method of conjugate gradients for lower bounding the quadratic form appearing in the log marginal likelihood. This approach has several advantages. First, it provides a natural stopping criterion for CG and allows the user to reuse past CG solutions, resulting in very few iterations of CG to be run. Second, it provides an objective function that is a deterministic lower bound, which was shown to resolve some of the optimization difficulties that can emerge with the method used by \citet{gardner2018gpytorch}, both due to bias and gradients.

\Citet{pmlr-v139-potapczynski21a} proposed a variant of the algorithm used by \citet{gardner2018gpytorch}
that provides unbiased estimates of the log marginal likelihood and its gradients using a ``Russian Roulette'' estimator \citep{kahn1955use}. This comes at the cost of additional variance. Both \citet{pmlr-v139-artemev21a} and \citet{wenger2021reducing} provide evidence that variance can be detrimental to the optimization of the log marginal likelihood.  \Citet{wenger2021reducing} proposed using a preconditioner as a control variate when performing stochastic Lanczos quadrature to reduce the variance of the estimator. %We avoid the use of a ``Russian Roulette'' estimator while obtaining nearly as strong guarantees on our estimate of the objective as \citet{pmlr-v139-potapczynski21a} by adaptively choosing the amount of computation in each evaluation of the estimator.

Our work has two main differences from \citet{gardner2018gpytorch}. First, where they stop running the Krylov subspace methods based on the norm of the residual produced by the method of conjugate gradients, we rely on upper and lower bounds on an unbiased estimate of the log marginal likelihood. By doing this, we can be sure that more iterations of the Krylov subspace method would not substantially improve our estimate of the log marginal likelihood. Second, we use a unified objective and gradient function, instead of approximating both separately. The main reason for previous work avoiding such an approach is the computational cost of differentiating through the iterative solver. We avoid this by instead viewing the iterative solver as outputting auxiliary parameters, and differentiate through a bound based on these parameters, ignoring the dependence of the auxiliary parameters on hyperparameters. This can be thought of as essentially a block-updating procedure where we use gradient descent to update hyperparameters and a Krylov subspace method to update auxiliary parameters.

\section{Background}
We focus on the problem of Bayesian regression with a Gaussian process prior and additive, homoscedastic Gaussian noise. 

\subsection{Gaussian process regression and empirical Bayes}
We assume a dataset, $D=\{(x_i,y_i)\}_{i=1}^n$ with $y_i \in \R$ has been observed. The prior is a mean zero Gaussian process with prior covariance function $k_{\theta}$,  and $\theta$ are model hyperparameters that determine, for example, the shape and scale of the covariance function. Further, we denote the variance of the additive noise model by $\sigma^2$. We let $\nu=\{\theta,\sigma^2\}$. Let $K=K_{\nu}$ denote the $n \times n$ matrix with entries $[K_\nu]_{i,j} = k(x_i,x_j) + \delta_{i,j} \sigma^2$ for $1 \leq i,j \leq n$. The marginal likelihood of the data for this model is given by,
\begin{align}\label{eqn:lml}
    \mathcal{L}(\nu) = - c - \frac{1}{2}\log\det(K_{\nu}) - \frac{1}{2} y^\top K_{\nu}^{-1}y,
\end{align}
where $c=-\frac{n}{2}\log 2\pi$, and $y \in \R^n$ denotes the vector formed by stacking the $y_i$. We are interested in selecting $\nu$ via empirical Bayes, which involves maximizing $\mathcal{L}(\nu)$ with respect to $\nu$. Finding a global optimum of $\mathcal{L}(\nu)$ is generally difficult, but gradient-based local optimization of $\mathcal{L}(\nu)$ has been shown to be a successful heuristic for selecting settings of $\nu$ that balance data fit with model complexity \citep[Chapter 5]{rasmussen2006gaussian}. However, computing $\mathcal{L}(\nu)$ and $\nabla_{\nu}\mathcal{L}(\nu)$ directly is typically done with a Cholesky factorization of $K_{\nu}$ which results in a cost that scales cubically with the number of observed datapoints.

\subsection{The method of conjugate gradients and stochastic Lanczos quadrature}

\paragraph{The method of conjugate gradients}

The term $y^\top K_{\nu}^{-1}y$ in \cref{eqn:lml} can be computed with a single inner product if we can find a solution to the system of equations,
$K_{\nu}v=y$. Given an initial guess $v_0$, the method of conjugate gradients \citep{hestenes1952methods} constructs a sequence of vectors $\{v_0,v_1,\dotsc v_n\}$ such that $v_n$ is a solution to this system of equations. Computing a new $v_i$ involves a single matrix-vector multiplication, as well as some vector-vector operations. Importantly, for many systems of equations the $v_i$ converge rapidly to the solution of the system of equations, leading to good approximations to $y^\top K_{\nu}^{-1}y$, without running $n$ iterations. This can lead to large computational savings as compared to performing Cholesky decomposition. \Citet{Gibbs97efficientimplementation} observed that for any $v \in \R^n$, 
\begin{align}\label{eqn:cg-bounds}
 2r^\top v + v^\top K_{\nu} v   \leq y^\top K_{\nu}^{-1}y \leq \frac{r^\top r}{\sigma^2} + 2r^\top v + v^\top K_{\nu} v
\end{align}
with $r=y-K_{\nu}v$ and equality holding once CG has converged. If the upper and lower bounds in \cref{eqn:cg-bounds} are close, we can be sure that the method has returned an accurate solution. \Citet{pmlr-v139-artemev21a} generalized this approach by suggesting a tighter upper bound depending on a low-rank approximation to the kernel matrix, and proposed differentiating through this bound directly (instead of estimating the gradients separately). This leads to a sort of "block" maximization procedure, where $v$ is viewed as an auxiliary variable. CG optimizes the bound with respect to $v$ and is alternated with optimization with respect $\nu$ using a gradient-based method. We follow this approach for estimating $y^\top K_{\nu}^{-1}y$, as it already provides a stopping criterion directly related to the objective.

\paragraph{Stochastic Lanczos Quadrature}

The application of stochastic Lanczos quadrature to estimating the log determinant of a matrix was introduced in \citet{ubaru2017fast}. The idea is to note that $\log \det(K_{\nu}) = \mathrm{tr}(\log K_{\nu})$. Hutchinson's trace estimator \citep{hutchinson1989stochastic} then yields,
\begin{align}
    \log \det(K_{\nu}) = \mathbb{E}[z^\top \log K_{\nu} z],
\end{align}
where $z$ is a random vector in $\mathbb{R}^n$ satisfying $\mathbb{E}[zz^\top]=I$. Lanczos quadrature provides a method for estimating expressions of the form $z^\top f(K_{\nu}) z$, by writing this as a Riemann-Stieltjes integral and using a Gauss quadrature rule. The nodes and weights of this rule can be computed through the Lanczos algorithm  \citep{golub1969calculation}. Moreover, as noted in \citet[Theorem 3.2]{golub1994matrices}, for a smooth function $f$ with negative even derivatives this results in a lower bound on the desired quantity. \Citet{10.2307/2028604} showed that with minor modifications, a Gauss-Radau quadrature rule can be used (i.e.~a quadrature with one prescribed node). If the odd derivatives of $f$ are positive, and the prescribed node is at least as large as the largest eigenvalue of $K_{\nu}$, the Gauss-Radau rule results in an upper bound on the desired quantity \citep[Theorem 3.2]{golub1994matrices}.

\Citet{pmlr-v48-lig16} previously made similar observations regarding a retrospective error analysis for kernel methods based on Gauss and Gauss-Radau rules, and provided guarantees about the quality of the resulting upper and lower bounds. While they noted that this sort of approach could potentially be applied to Gaussian processes, they did not consider the application to approximate empirical Bayes with Gaussian processes. This has some unique challenges due to the need to compute gradients of the estimator, that we investigate in this note. \Citet{pmlr-v139-potapczynski21a} observed that the Gauss quadrature estimator results in an upper bound on the log-determinant, and used this to heuristically explain the effect of the bias introduced to GPR by a lack of convergence.

\paragraph{Method for estimating the LML and its gradients}

Our approach is similar to the one in \citet{pmlr-v139-artemev21a}, but replacing the deterministic log-determinant estimator they use with an estimator based on stochastic Lanczos quadrature. In particular, our objective function is
\begin{align}\label{eqn:objective}
   \tilde{L}(\nu) =  c-\frac{1}{2s}\sum_{i=1}^sz_i^\top T_i \log(T_i^{\top}K_{\nu}T_i) T_i^\top z_i - \frac{1}{2} \left(\frac{r^\top r}{\sigma^2} + 2r^\top v + v^\top K_{\nu} v\right). 
\end{align}
where the $T_i$ are $n \times t$ auxiliary matrices subject to the constraint $T_i^\top T_i= I_t$ and $z_i$ is in the column space of $T_i$, $v$ is a vector in $\mathbb{R}^n$ and $r=y-K_{\nu}v$.  In \cref{app:lower-bound} we show that $\tilde{L}(\nu) \leq L(\nu)$ and is exact when $v = K_\nu^{-1}y$ and $T^\top_i T_i = I_n$.

 We select $T_i$ to be the orthogonal basis for a $t$ dimensional Krylov subspace constructed by the Lanczos algorithm with initial vector $z_i/\|z_i\|$, in which case in exact arithmetic $T^\top_i K_\nu T_i$ is tridiagonal and the first term in \cref{eqn:objective} is the same as the stochastic Lanczos Gauss quadrature estimate of $\log \det (K_\nu)$ proposed by \citet{ubaru2017fast}. The vector $v$ is updated with CG as in \citet{pmlr-v139-artemev21a}. We compute the matrix logarithm via eigendecomposition. In practice, we also make use of a low-rank preconditioner, and use the tighter bound on the quadratic term proposed in \citet{pmlr-v139-artemev21a} based on such a preconditioner.

We assess lower and upper bounds on \cref{eqn:objective} during each iteration of the Lanczos algorithm and CG, and stop running these methods when the bias is probably less than $\epsilon$, where $\epsilon>0$ is a user-specified parameter. We directly differentiate through \cref{eqn:objective}, ignoring the dependence of $T$ and $v$ on $\nu$ that is implicit in the method we used to select them. As such, like in \citet{pmlr-v139-artemev21a}, our optimization procedure can be seen as a form of block updating of model hyperparameters and auxiliary parameters.

\section{Preliminary experimental results for Barely Biased GP (BBGP)}

We present results for BBGP on three UCI datasets \citep{Dua:2019}. In experiments, we use a Mat\'ern 3/2 kernel with a separate lengthscale per input dimension. Details of the normalization of the data and initialization of hyperparameters are given in \cref{app:experiments}. We select hyperparameters for BBGP with Adam \citep{kingma2014adam} using initial learning rate $0.1$. We compare against the implementations of CGLB and SGPR provided in GPflow \citep{GPflow2017}. Inducing points are initialized following the `greedy' procedure discussed in \citet{JMLR:v21:19-1015} then optimized jointly with hyperparameters using L-BFGS \citep{liu1989limited}. We also compare our method against the conjugate-gradient based implementation provided in \citet{gardner2018gpytorch} (`Iterative GP'). We optimize hyperparameters of this method with the procedure described in \citet{pmlr-v139-artemev21a}.

\begin{figure}[htb]
    \begin{minipage}[t]{0.495\linewidth}
        \centering
        \includegraphics[width=\textwidth]{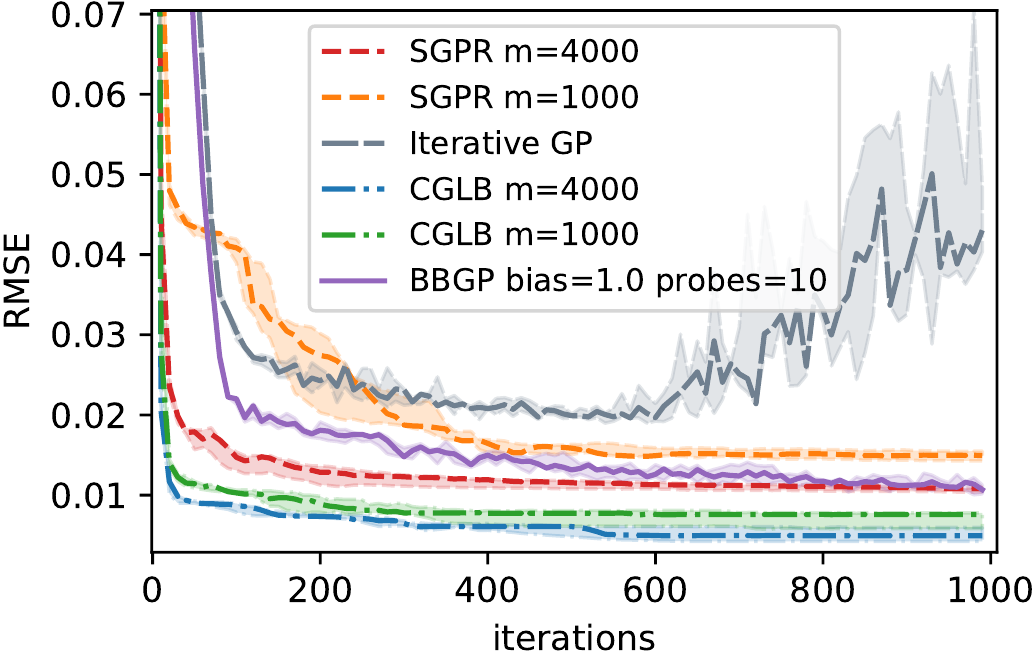}
    \end{minipage}
    % \hspace{0.5cm}
    \hfill
    \begin{minipage}[t]{0.495\linewidth}
        \centering
        \includegraphics[width=\textwidth]{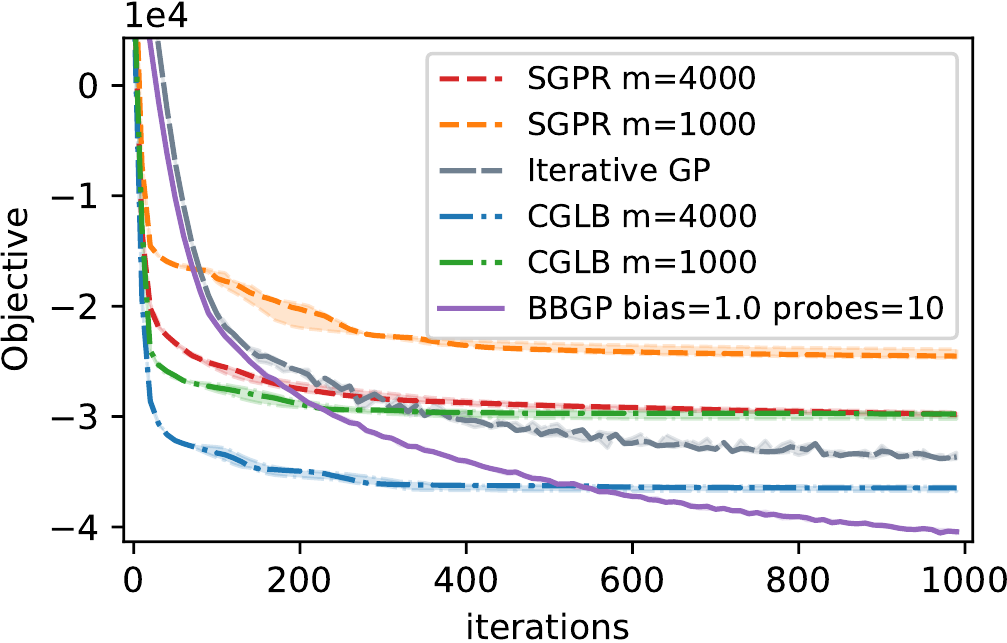}
    \end{minipage}
    \caption{Model performance on testing data and objective (an estimation of negative log marginal likelihood) traces over optimization steps for \texttt{bike} dataset.}
    \label{fig:bike-rmse-nlml}
\end{figure}

In the conducted experiments we record RMSE on testing dataset splits to compare performances of models. BBGP showed promising results and dominated CGLB in terms of RMSE metric in two out of three datasets (\cref{app:experiments}, \cref{fig:elevators-rmse-nlml,fig:pol-rmse-nlml}) and outperformed SGPR and Iterative GP on all datasets. BBGP showed robust behaviour during stochastic optimization in intrinsically low noise data such as \texttt{bike} and \texttt{poletele} as a counter to Iterative GP. Iterative GP exhibited divergence in performance and overestimated LML for datasets with low noise \citep[as discussed in][]{pmlr-v139-artemev21a}.
We do not provide an estimation of GP variance and therefore we did not compute log predictive density (LPD). A user can leverage any available GP variance estimator (e.g.~\citet{pmlr-v5-titsias09a, pmlr-v80-pleiss18a}).

BBGP has two parameters to tune: the bias parameter $\epsilon > 0$ and the number of probe vectors $s$ in \cref{eqn:objective}.  We investigated the effect of large bias $10.0$ and $100.0$, additionally to the default $1.0$ value together with small number of probe vectors $s=1$. Our study (\cref{app:experiments}, \cref{fig:bias-study}) showed that $s=1$ does not increase the variance dramatically, and yields similar performance. We note that $10$ or even $100$ nats correspond to a very small change in hyperparameters when the likelihood noise is very small. With larger values of $\epsilon$ we observed that number of steps required to reach the desired level of bias decreases and the stability of the RMSE and LML estimation during training decreases, but stays within an acceptable range (\cref{app:experiments}, \cref{fig:bias-study}). Larger $\epsilon$ and $s = 1$ reduced training time significantly, however, it was not enough to compete with other methods; BBGP took about 4 times as long to run as CGLB on \texttt{bike} and 2 times as long to run on \texttt{poletele}. 

% * Robustness to data with intrinsic low noise
% * Performance in terms of RMSE and the reason for not running negative predictive density.
% * Challenges with speed
% *  

% \subsection{Elevators}

% \subsection{Bike}

\section{Current obstacles to the proposed approach}

While we believe that the approach of directly relating the stopping criterion for a Krylov subspace method with the objective function provides an elegant alternative to existing criterion, our current implementation is not competitive with other methods in terms of computational cost per iteration. In order to perform full re-orthogonalization and to compute the gradients of the approximate objective, we store intermediate vectors produced during CG. These algorithmic steps cause memory issues that we have not yet been able to circumvent. As a result, for datasets with a very small likelihood variance, we set a maximum number of iterations to run, which loses much of the elegance of our proposed approach. Additionally the computational time to estimate the LML for these datasets can be much higher than with other approximate methods.
%%%%%%%%%%%%%%%%%%%%%%%%%%%%%%%%%%%%%%%%%%%%%%%%%%%%%%%%%%%%
\bibliographystyle{apalike}
\bibliography{icbinb2021}
\clearpage
\appendix

\section{Proof of Lower bound}\label{app:lower-bound}

As \cref{eqn:cg-bounds} has been previously established \citep{Gibbs97efficientimplementation}, it suffices to show that $\log\det(K) \leq \mathbb{E}[z^\top T\log(T^{\top}K_{\nu}T)T^\top z]$ for such that $\mathbb{E}[zz^\top]=I$ and $T$ sat. Writing $\log \det (K) = \mathrm{tr}(\log(K)) = \mathbb{E}[z^\top \log(K) z]$ it suffices to show that that for any $z\in \R^n$, and $T$ satisfying the assumptions, $z^\top\log(K_{\nu})z \leq z^\top\log(T^{\top}K_{\nu}T)z$.

Since$T^\top T=I_t$, then $TT^\top = TT^\top TT^\top$, and $(TT^\top)^\top=TT^\top$ i.e.~$TT^{\top}$ is a Hermitian projection. We define $P=TT^\top$. Since $z$ is by assumption in the column span of $T$ we have $Pz = z$. Hence,
\begin{align}
    z^\top \log(K)z = z^\top P\log(K)Pz = z^\top TT^\top\log(K)TT^\top z.
\end{align}
Define $w = T^\top z \in \R^k$, so this can be rewritten
\begin{align}
    z^\top \log(K)z = w^\top T^\top\log(K)T w.
\end{align}
It therefore suffices to prove the following proposition:

\begin{proposition}\label{prop:log-psd}
Let $K \in \mathbb{R}^{n \times n}$ symmetric, positive definite. Let $T \in \mathbb{R}^{n \times t}$ such that $T^{\top}T = I_t$. Then,
\[
\log(T^{\top}KT) - T^{\top}\log(K)T
\]
is symmetric positive semi-definite.
\end{proposition}
We will make use of several lemmas in proving \cref{prop:log-psd}.
\begin{lemma}[]\label{lem:integral-psd}
Let $A(s) \in \R^{n \times n}$ be a matrix parameterized by $s \in (0, \infty)$ such that, for each $s$, $A(s)$ is symmetric positive semi-definite and $A(s)$ is integrable. Then, $\int_{s=0}^\infty A(s) ds$ is symmetric, positive semi-definite.
\end{lemma}
\begin{proof}
Let $z \in \R^n$ arbitrary. By linearity of integral, we have
$z^{\top} \left(\int_{s=0}^\infty A(s)ds\right) z = \int_{s=0}^\infty z^{\top} A(s) z ds$. The integrand is non-negative, from which the result follows.
\end{proof}
\begin{lemma}[\cite{3950744}]\label{lem:integral-rep}
Let $A$ be a square matrix with no eigenvalues on $(-\infty,0]$, then
\begin{equation}
    \log(A) = \int_{s=0}^\infty \frac{1}{1+s}I - (A+sI)^{-1}ds.
\end{equation}
\end{lemma}
See the cited math stack-exchange article for a proof.

\begin{lemma}[Special Case of Sherman-Morrison-Woodbury Lemma]\label{lem:woodbury}
\begin{equation}
    (UV + sI)^{-1} = \frac{1}{s}\left(I-U(VU+sI)^{-1}V\right)
\end{equation}

\end{lemma}
\begin{proof}[Proof of \cref{prop:log-psd}]

By assumption $K$ is symmetric positive definite so its has no eigenvalues in $(-\infty, 0]$. $T^{\top}KT$ is positive semi-definite, and by Cauchy's interlacing theorem, its eigenvalues  are contained in the convex hull of the eigenvalues of $K$, hence it is positive definite and both logarithms are well-defined.

We use \cref{lem:integral-rep} to represent $\log(T^\top K T)$ and $T^\top \log(K) T$:
\begin{align}
    \log(T^\top K T) &= \int_{s=0}^\infty \frac{1}{1+s}I_t - (T^\top K T +sI_t)^{-1}ds.
\end{align}
and
\begin{align}
    T^\top \log(K)T &= T^\top \left(\int_{s=0}^\infty \frac{1}{1+s}I_n - (K+sI_n)^{-1}ds \right)T \\
    & = \int_{s=0}^\infty \frac{1}{1+s}T^\top T - T^\top(K+sI_n)^{-1}Tds \\
    &=\int_{s=0}^\infty \frac{1}{1+s}I_t - T^\top(K+sI_n)^{-1}Tds.
\end{align}
Hence, 
\begin{align}
    \log(T^\top KT) - T^\top \log(K) T = \int_{s=0}^\infty \left(T^\top(K+sI_n)^{-1}T-(T^\top KT+sI_t)^{-1}\right)ds.
\end{align}    
By \cref{lem:integral-psd}, it suffices to show that $T^\top (K+sI_n)^{-1}T-(T^\top KT+sI_t)^{-1}$ is positive definite for all $s$.

Applying \cref{lem:woodbury} to the first matrix,
\begin{align}
    T^\top (K^{1/2}K^{1/2}+sI_n)^{-1}T^\top &= \frac{1}{s}T^\top\left(I - K^{1/2}(K+sI_n)^{-1}K^{1/2}\right)T \\
    & = \frac{1}{s}\left(I_t - T^\top K^{1/2}(K+sI_n)^{-1}K^{1/2}T\right)
\end{align}
Applying \cref{lem:woodbury} to the second matrix,
\begin{align}
    (T^\top KT+sI_t)^{-1} &= \frac{1}{s}\left(I_t - T^\top K^{1/2}(K^{1/2}TT^\top K^{1/2}+sI_n)^{-1}K^{1/2}T\right).
\end{align}
Hence the difference is, 
\begin{align}
   &T^\top(K+sI_n)^{-1}T-(T^\top KT+sI_t)^{-1} \\&= \frac{1}{s}\left(T^\top K^{1/2}(K^{1/2}TT^\top K^{1/2}+sI_n)^{-1}K^{1/2}T - T^\top K^{1/2}(K+sI_n)^{-1}K^{1/2}T\right) \\ 
   &= \frac{1}{s}T^\top K^{1/2}\left((K^{1/2}TT^\top K^{1/2}+sI_n)^{-1} - (K+sI_n)^{-1}\right) K^{1/2}T \label{eqn:final-psd}
   .
\end{align}
But $K^{1/2}TT^\top K^{1/2}\prec K$, so the matrix in \cref{eqn:final-psd} is positive semi-definite.

We note that in the special case when $T$ is obtained by Gauss Lanczos quadrature, the validity of the lower bound can be shown by properties of quadrature rules, see \citet{golub1994matrices, pmlr-v48-lig16, pmlr-v139-potapczynski21a}.
\end{proof}

\section{Experiments}
\label{app:experiments}

In experiments we compare the presented method (BBGP) with sparse methods \citep{pmlr-v5-titsias09a} (SGPR), and iterative based methods from \citet{pmlr-v139-artemev21a} and \citet{gardner2018gpytorch} (CGLB and Iterative GP respectively). We chose three UCI datasets (\texttt{elevators}, \texttt{poletele} and \texttt{bike}) to investigate the performance of our method. The prior constant mean is initialized at $0$, the likelihood variance is initialized with $1.0$, and lengthscales and the scaling parameter of the Mat\'ern 3/2 kernel are initialized at $1.0$, and a separate lengthscale is used for each input dimension. For all constrained positive parameters we set the lower bound at $1\mathrm{e}{-6}$. We used CGLB and SGPR models with $m=1000$ and $m=4000$ inducing points, which we initialized with greedy selection method. Each dataset was randomly split with 2/3 proportion for the training subset and the rest for testing points. All experiments were run for 5 different seeds, and in graphs we report lower and upper quantiles (shaded region) and the median for RMSE and objective.

\begin{figure}[htb]
    \begin{minipage}[t]{0.495\linewidth}
        \centering
        \includegraphics[width=\textwidth]{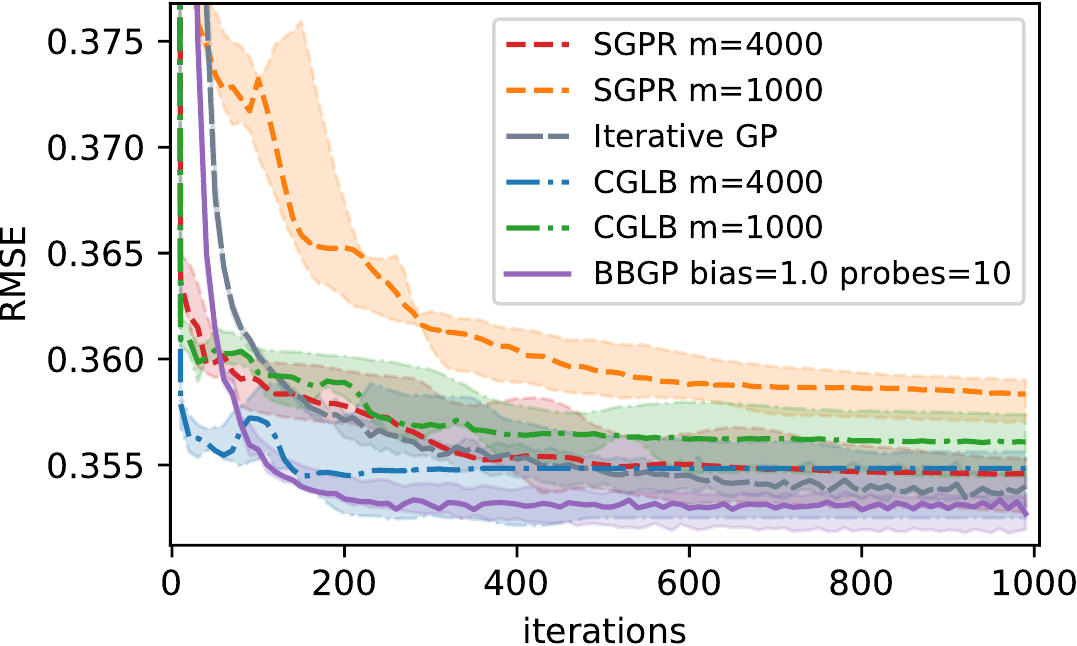}
    \end{minipage}
    % \hspace{0.5cm}
    \hfill
    \begin{minipage}[t]{0.495\linewidth}
        \centering
        \includegraphics[width=\textwidth]{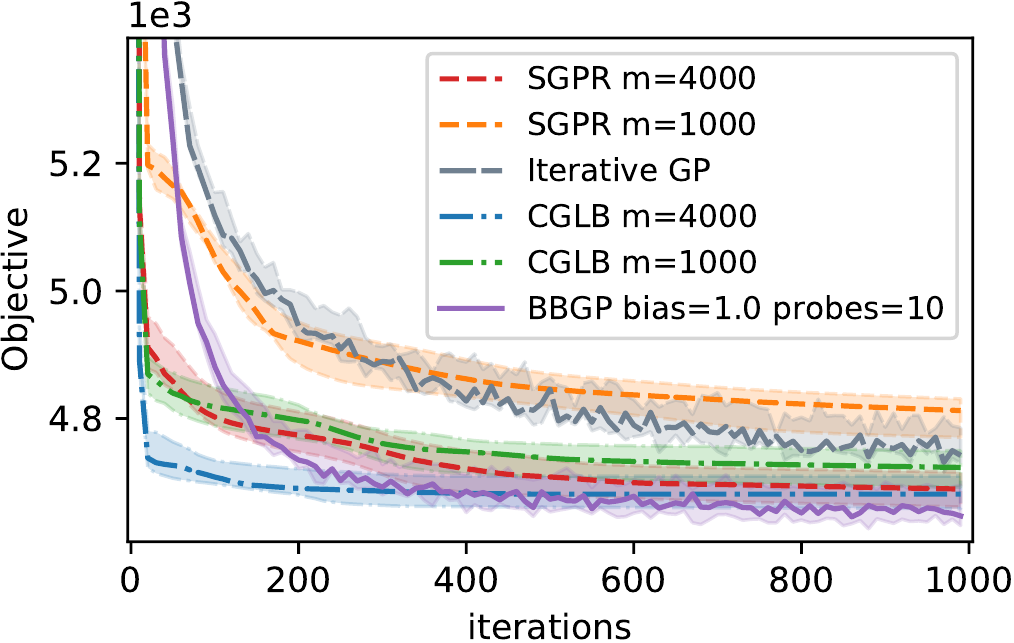}
    \end{minipage}
    \caption{Model performance on testing data and objective (an estimation of negative log marginal likelihood) traces over optimization steps for \texttt{elevators} dataset.}
    \label{fig:elevators-rmse-nlml}
\end{figure}

\begin{figure}[htb]
    \begin{minipage}[t]{0.495\linewidth}
        \centering
        \includegraphics[width=\textwidth]{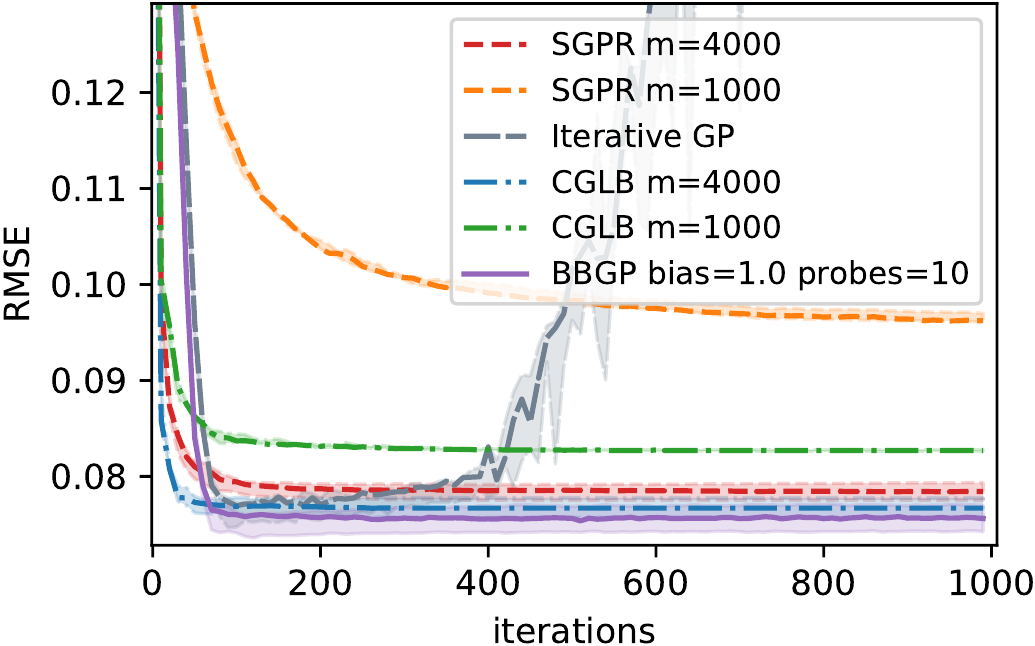}
    \end{minipage}
    % \hspace{0.5cm}
    \hfill
    \begin{minipage}[t]{0.495\linewidth}
        \centering
        \includegraphics[width=\textwidth]{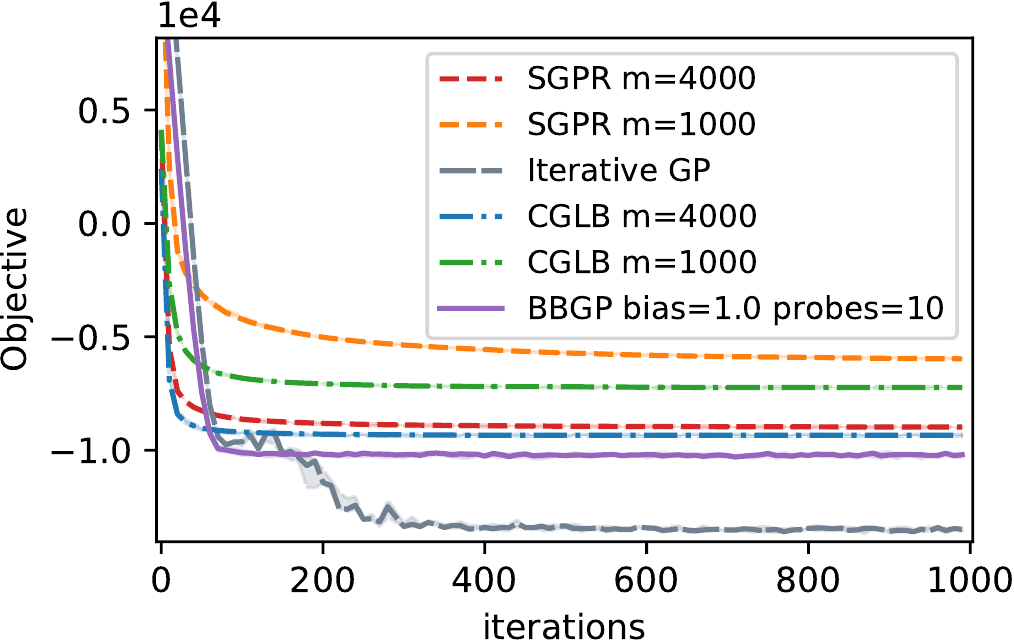}
    \end{minipage}
    \caption{The plot on the right shows RMSE metric for BBGP, SGPR, CGLB and Iterative GP models with different configurations. The plot on the left contains graphics for LML estimation (optimization objective) for the same models. \texttt{poletele} is a low noise dataset where noise level $\approx 1\mathrm{e}{-4}$. BBGP learns model hyperparameters without any visible artifacts in RMSE and LML traces. Meanwhile, Iterative GP demonstrates signs of divergence and LML overestimation.}
    \label{fig:pol-rmse-nlml}
\end{figure}

\begin{figure}[htb]
    \begin{minipage}[c]{0.495\linewidth}
        \centering
        \includegraphics[width=\textwidth]{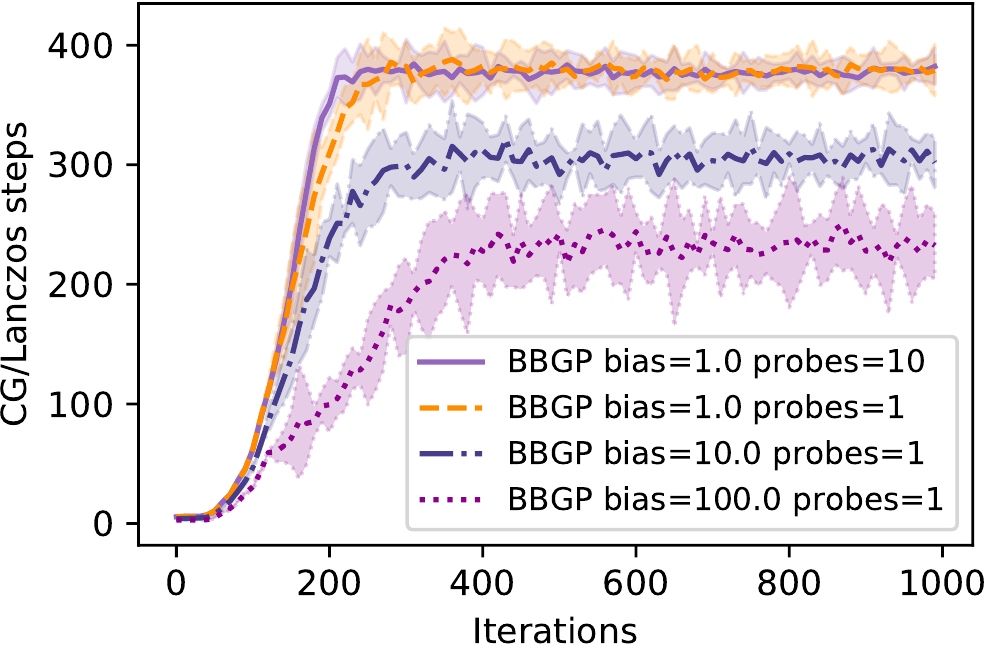}
    \end{minipage}
    % \hspace{0.5cm}
    \hfill
    \begin{minipage}[c]{0.495\linewidth}
        \centering
        \includegraphics[width=\textwidth]{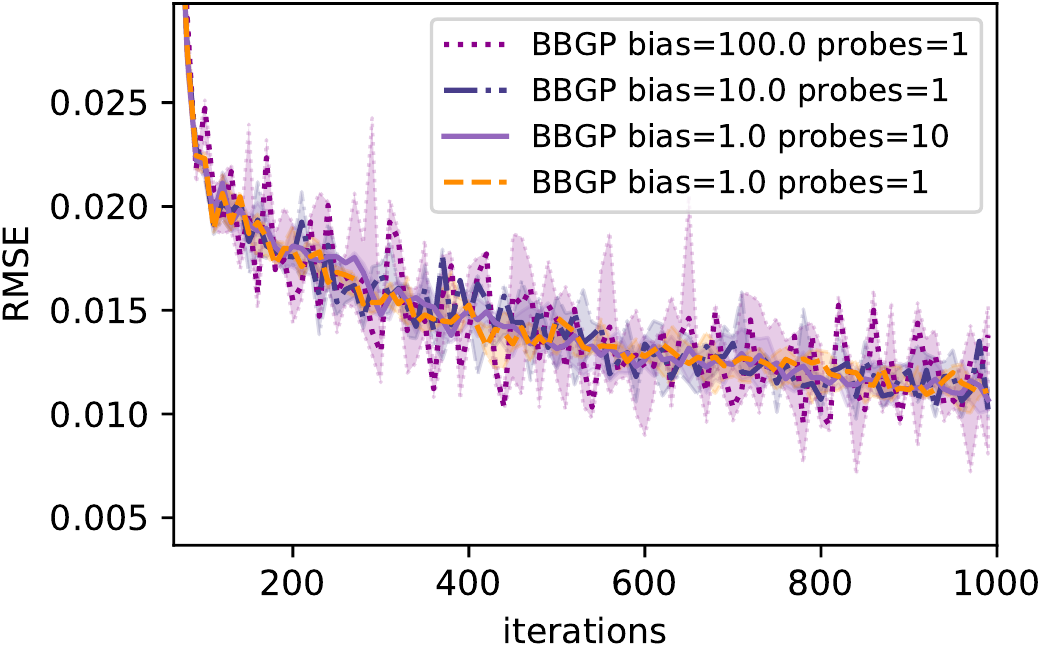}
    \end{minipage}
    \caption{On the left, the graph shows the number of steps spent per optimization iteration for different BBGP model configurations on the \texttt{bike}. The right plot contains RMSE metrics for the same models. Larger bias yields shorter CG/Lanczos runs without degradation in performance.}
    \label{fig:bias-study}
\end{figure}

\end{document}